\documentclass[journal]{IEEEtran}
\usepackage[utf8]{inputenc}
\usepackage{cite}
\usepackage{hhline}
\usepackage{amsmath}
\usepackage{amssymb}
\usepackage{graphicx}
\usepackage{bm}
\usepackage{multirow}
\usepackage{dcolumn}
\usepackage{hyperref}
\hypersetup{colorlinks=true, linkcolor=blue, citecolor=blue}

\newcommand{\ang}[1]{\left\langle #1 \right\rangle}

\newcommand{\var}{\operatorname{Var}}
\newcommand{\norm}[1]{\left\lVert #1 \right\rVert}
\newcommand{\opnorm}[1]{\norm{#1}^{\mathsf{op}}}

\newcommand{\C}{\mathbb{C}}
\newcommand{\E}{\mathbb{E}}

\newcommand{\R}{\mathbb{R}}
\newcommand{\Z}{\mathbb{Z}}
\newcommand{\ve}{\varepsilon}

\newcommand{\mbf}[1]{\mathbf{#1}}

\newcommand{\msf}[1]{\mathsf{#1}}

\usepackage{amsthm}
\newtheorem{thm}{Theorem}[section]
\newtheorem{lem}[thm]{Lemma}

\newtheorem{prop}[thm]{Proposition}

\begin{document}

\title{Polynomial Bounds for Learning Noisy Optical Physical Unclonable Functions and Connections to Learning With Errors}

\author{Apollo Albright, Boris Gelfand, and Michael Dixon
\thanks{Apollo Albright, Boris Gelfand, and Michael Dixon are with Los Alamos National Laboratory, Los Alamos, New Mexico (e-mail:  aalbright@lanl.gov; bgelfand@lanl.gov; mdixon@lanl.gov)}
\thanks{Apollo Albright is also with Reed College, 3203 SE Woodstock Blvd, Portland, Oregon 97202 USA}
}
\maketitle

\begin{abstract}
    It is shown that a class of optical physical unclonable functions (PUFs) can be learned to arbitrary precision with arbitrarily high probability, even in the presence of noise, given access to polynomially many challenge-response pairs and polynomially bounded computational power, under mild assumptions about the distributions of the noise and challenge vectors.
    This extends the results of Rh\"uramir et al. (2013), who showed a subset of this class of PUFs to be learnable in polynomial time in the absence of noise, under the assumption that the optics of the PUF were either linear or had negligible nonlinear effects.
    We derive polynomial bounds for the required number of samples and the computational complexity of a linear regression algorithm, based on size parameters of the PUF, the distributions of the challenge and noise vectors, and the probability and accuracy of the regression algorithm, with a similar analysis to one done by Bootle et al. (2018), who demonstrated a learning attack on a poorly implemented version of the Learning With Errors problem.
\end{abstract}

\section{Introduction}
The security of a cryptographic system depends on the security of the keys and encryption mechanisms it uses.
Traditional cryptographic systems that store sensitive or proprietary information in non-volatile memory are susceptible to having this information copied to a malicious machine.
One solution to this problem is to use a physical unclonable function (PUF) \cite{Pappu_2001,CDDD_2002}.
A PUF is a type of one-way physical system characterized by instance-specific random physical properties arising from manufacturing process variations.
A PUF can be probed or challenged with external stimuli to give specific responses, which depend on random variations during the manufacturing process and are ideally impossible to predict or invert without directly interrogating the PUF.
PUFs are often characterized by some form of randomness or disorder inherent in the manufacturing process, which is ideally impossible for any party to reproduce, or clone, exactly.
This unclonability property makes PUFs ideal for technology protection, anti-tamper attestation, and cryptographic protocols such as key generation that require an entropy source for secure random number generation protocols since they cannot be directly copied like digital keys or code stored in non-volatile memory \cite{Pappu_2001, CDDD_2002, SD_2007}.

By sending the PUF a sequence of challenges and checking that it returns the correct responses, one can verify the PUF's integrity.
One measure of the strength of a PUF is the number of challenge-response pairs (CRPs), which are unique pairs $(C,R)$ of challenges $C$ and responses $R$.
A PUF in which the number of CRPs scales polynomially with a security parameter $n$ (which may be the physical size or number of inputs of the system) is classified as ``weak'' since its behavior can be fully determined by polynomial-time read-out attacks, whereas a PUF that has exponentially many CRPs is classified as ``strong'' since it is not vulnerable to these sorts of brute-force attacks \cite{MBW+_2019}.

Many current PUF designs are implemented in electronic circuits and use signal race conditions set by the inherent randomness in silicon manufacturing \cite{MBW+_2019}.
Examples of silicon-based PUFs include the Arbiter PUF \cite{SD_2007,GCD+_2002,LDG+_2004}, Ring Oscillator PUFs \cite{BNCF_2014}, and static random-access memory (SRAM) PUFs \cite{GKST_2007,HBF_2008,MTV_2008}.
Many of these designs, such as the Arbiter PUF and its variants, have been demonstrated to be machine learnable \cite{DLG+_2005, RSS+_2010, TLH+_2015, GKJ+_2015, GTSS_2015, GTFS_2016, GTS_2016, GTFS_2017, Ganji_2018, CMH_2020}.
Once an adversary has a model of the PUF, they can encode it in a separate chip to create a functional copy of it.
In addition, physical clones of SRAM PUFs were created using a focused ion beam circuit edit in \cite{HBNS_2013}, further limiting the application of silicon PUFs that rely on race conditions for implementing secure and unclonable physical cryptographic protocols.

Optical PUFs, first introduced in \cite{Pappu_2001,PRTG_2002}, were one of the first suggested PUF designs.
Optical PUFs consist of an optical medium, typically some kind of resin, with strongly scattering material, such as microscopic glass beads, randomly distributed within.
When coherent laser light hits the medium, it undergoes many scattering events as it passes through the sample, resulting in a noisy image called a speckle pattern on the opposite side.
A challenge for the optical PUF therefore consists of the position and angle of incidence of the laser source, and the response is an image of the speckle pattern.
While the optical PUFs presented in \cite{Pappu_2001, PRTG_2002} were experimentally shown to be resistant to modeling attacks by Support Vector Machines (SVMs) \cite{RHU+_2013}, they are still classified as weak PUFs since they suffer from a polynomially bounded set of CRPs due to the optical structure having nonzero correlation lengths and angles \cite{Pappu_2001}, making very small changes in the orientation of the incident laser result in highly correlated speckle patterns \cite{RHU+_2013}.
The correlation lengths and angles can be reduced greatly by using nonlinear optical media \cite{Pappu_2001, SM_2000}; however the number of CRPs is still polynomially bounded by the precision of the laser alignment system.
Because of this polynomial bound on the number of available CRPs, an adversary can efficiently generate a model of the PUF just by enumerating every possible CRP, regardless of measurement noise.
Furthermore, the original optical PUFs require a very precise token positioning system and are prone to misalignment error, making them somewhat unreliable.

These issues were addressed in \cite{RHU+_2013} with the introduction of integrated optical PUFs.
In the original non-integrated optical PUFs, the relative position of the laser and the scattering medium can be varied as part of the challenge.
In contrast, an integrated optical PUF fixes the relative positions of the laser, the PUF, and the camera.
In order to input different challenges, the authors of \cite{RHU+_2013} propose to send the incoming laser beam through a collimating lens and a spatial light modulator, such as a liquid-crystal display (LCD) mask, allowing parts of the PUF's surface to be selectively illuminated (Fig. \ref{fig:PUF_schematic}).
Thus, a challenge for the integrated PUF in \cite{RHU+_2013} consists of a specific image on the mask, and the response is the corresponding speckle pattern.
Since the number of mask images is exponentially large in the number of pixels, optical PUFs with a mask have exponentially many CRPs, and are thus classified as ``strong''.
Since integrated optical PUFs do not have any moving parts, they are not as reliant on the exact position and angle of the incident laser and are less susceptible to environmental changes than the ones in \cite{Pappu_2001,PRTG_2002}.

\begin{figure}[!t]
    \centering
    \includegraphics[width=0.8\linewidth]{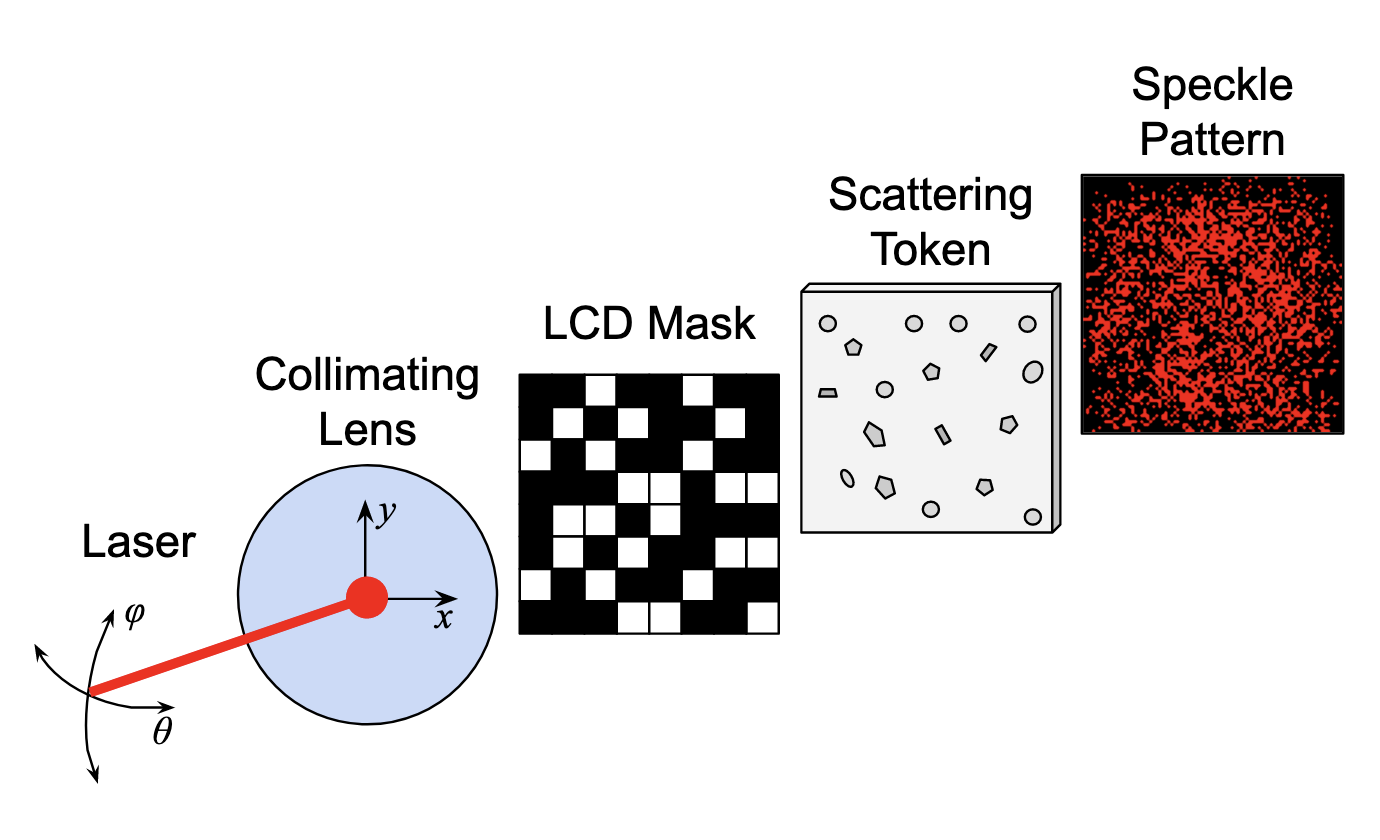}
    \vspace{-0.25cm}
    \caption{A schematic of an optical PUF with a mask.
    By varying the laser's position $(x,y)$ and angle $(\theta,\varphi)$ relative to the scattering token and selecting which blocks of the mask are transparent, one can control which areas of the scattering pattern are illuminated.
    The resulting speckle pattern can be recorded by a camera.
    In the integrated optical PUF design proposed in \cite{RHU+_2013}, the laser's position and angle are fixed, and in the original optical PUF of \cite{Pappu_2001}, which did not feature a mask, the laser hits the scattering token directly.}
    \label{fig:PUF_schematic}
\end{figure}

It was shown in \cite{RHU+_2013} that, in the absence of external noise, integrated optical PUFs using a mask and scattering media with linear optical properties are susceptible to linear regression algorithms since the electric field amplitudes of speckle patterns from different challenges add together linearly.
By generating a basis of the challenge space, it is possible for an adversary to predict the behavior of a linear combination of these basis challenges since the corresponding response will be the same linear combination of the responses.

In this paper, we extend this result to show that optical PUFs with a mask are also learnable in polynomial time when the external noise either has a bounded magnitude or if it follows a subgaussian distribution.
Our analysis based on an a proof in \cite{BDE+_2018} for the solvability of the ``Integer Learning with Errors'' problem, an easier variant of the Learning with Errors (LWE) problem that does not use modular reduction in the field $\Z/p\Z$.
More specifically, in Section \ref{subsec:physics} we examine the physics of the PUF and show that, within a linear optical regime, the responses can be written as a linear function of the challenges.
In Section \ref{sec:learning}, we describe how to reduce the problem of learning an integrated optical PUF (or equivalently, a particular challenge position of a non-integrated optical PUF) with a mask from noisy CRPs to the equivalent problem of solving a polynomially large system of noisy equations.
We prove a polynomial bound for the number of samples required to learn the PUF, based on the number of pixels in the LCD mask, the number of pixels in the output, the distribution of the challenge and noise vectors, the accuracy to which the PUF should be learned, and the desired probability of learning.
We conclude Section \ref{subsec:linear_learning} by expressing this bound asymptotically in Eq. \ref{eq:puf_learnability} and the time complexity of the linear regression algorithm in Eq. \ref{eq:puf_time_complexity}, and we extend this result to include weakly nonlinear regimes in Section \ref{subsec:nonlinear_learning}.
The effects of Kerr nonlinearity on the resistance of optical PUFs to physical cloning attacks was discussed in \cite{Nikolopoulos_2022}, however to our knowledge there have been no studies on learning attacks of nonlinear optical PUFs.
Since the linear regression algorithm runs in polynomial time and produces, with arbitrarily high probability, an arbitrarily good approximation to the PUF, we know these types of optical PUFs are learnable under the probably approximately correct (PAC) framework, which has previously been used to demonstrate the learnability of various other PUF designs \cite{GTSS_2015,GTS_2016,GTFS_2016,GTFS_2017,Ganji_2018,CMH_2020}.
Table~\ref{tab:learnability_results} gives results from the literature as well as our contributions for the learnability of optical PUFs.

\begin{table*}[t]
\begin{center}
    \caption{Learnability results of optical PUFs.}
    \begin{tabular}{|c|c|c||c|c|c|c|c|c|}
    \hline
     \multirow{2}{*}{Design} & \multirow{2}{*}{Illumination} & \multirow{2}{*}{CRP Space} & \multicolumn{2}{c|}{Linear} & \multicolumn{2}{c|}{Weakly Nonlinear} & \multicolumn{2}{c|}{Strongly Nonlinear}\\
    \cline{4-9}
     & & & Noiseless & Noisy & Noiseless & Noisy & Noiseless & Noisy \\ \hhline{|=|=|=#=|=|=|=|=|=|}
     Integrated & No Mask & $1$ & trivial & trivial & trivial & trivial & trivial & trivial \\ \hline
     Non-Integrated & No Mask & $O(\msf{poly}(n))$ & \cite{RHU+_2013, Pappu_2001} & \cite{RHU+_2013, Pappu_2001} & Section \ref{subsec:non-integrated_learning} & Section \ref{subsec:non-integrated_learning} & Section \ref{subsec:non-integrated_learning} & Section \ref{subsec:non-integrated_learning} \\ \hline
     Integrated & Mask & $O(\exp(n))$ & \cite{RHU+_2013} & Section \ref{subsec:linear_learning} & Section \ref{subsec:nonlinear_learning} & Section \ref{subsec:nonlinear_learning} & ? & ? \\ \hline
     Non-Integrated & Mask & $O(\exp(n))$ & Section \ref{subsec:linear_learning} & Section \ref{subsec:linear_learning} & Section \ref{subsec:nonlinear_learning} & Section \ref{subsec:nonlinear_learning} & ? & ? \\ \hline
    \end{tabular}
    \label{tab:learnability_results}
\end{center}
\end{table*}

\section{Preliminaries}
\label{sec:prelim}
\subsection{Notation}
\label{subsec:notation}
For a vector $\mbf{x}\in\R^n$, the $p$-norm $\norm{\mbf{x}}_p$ of $\mbf{x}$, for $p\geq 1$ is given by $\norm{\mbf{x}}_p=(|x_1|^p+\cdots+|x_n|^p)^{1/p}$.
Unless otherwise stated, $\norm{\mbf{x}}$ will always refer to the Euclidean norm $\norm{\mbf{x}}_2$.
For a matrix $\mbf{A}\in\R^{m\times n}$, the operator norm $\opnorm{\mbf{A}}$ is given by
\[
\opnorm{\mbf{A}} \;=\; \sup_{\norm{\mbf{x}}\,=\,1}\norm{\mbf{A}\mbf{x}}.
\]
We denote the maximum real eigenvalue of a square matrix $\mbf{A}$ by $\lambda_{\max}(\mbf{A})$, and similarly $\lambda_{\min}(\mbf{A})$ denotes the minimum real eigenvalue.
The transpose of a matrix $\mbf{A}$ is written as $\mbf{A}^\msf{T}$.
With this in mind, the operator norm of $\mbf{A}$ can be expressed as its largest singular value,
\begin{equation}
    \label{eq:opnorm}
    \opnorm{\mbf{A}} \;=\; \sqrt{\lambda_{\max}\left(\mbf{A}\mbf{A}^\msf{T}\right)}.
\end{equation}
We write $X\sim\chi$ to say a random variable $X$ is sampled according to a distribution $\chi$.
The expectation of $X$ is denoted $\E[X]$ and its variance $\var(X)=\E[X^2]-\E[X]^2$.
We denote by $\Pr[Y]$ the probability of event $Y$.

\subsection{Subgaussian Probabilitiy Distributions}
\label{subsec:prob_lemmas}

A variable $X$ is called $\tau$-subgaussian for some $\tau>0$ if for all $s\in\R$,
\[
\E\left[\exp(sX)\right]\leq\exp\left(\frac{\tau^2s^2}{2}\right).
\]
Subgaussian random variables are very useful for our analysis since they are subject to very strong tail bounds (at least as strong as those for a Gaussian distribution).
The following lemmas describe useful properties of subgaussian distributions, and they will be used in Section \ref{sec:learning} to bound the error an adversary would have when trying to learn the behavior of the PUF.
The proofs for Lemmas \ref{lem:subgaussian_properties}, \ref{lem:subgauss_def}, \ref{lem:subgauss_sum}, and \ref{lem:subgaussian_matrix} can be found in \cite{BDE+_2018}.

\begin{lem}[\cite{BLM_2013}, Lemma 2.2]
\label{lem:hoeffding}
Any distribution over $\R$ with mean zero and supported over a bounded interval $[-a, a]$ is $a$-subgaussian.
\end{lem}

\begin{lem}[\cite{BDE+_2018}, Lemma 2.4]
\label{lem:subgaussian_properties}
A $\tau$-subgaussian random variable $X$ has the following properties:
\[
\E[X]=0 \quad\textrm{and}\quad \E[X^2]\leq\tau^2.
\]
\end{lem}

\begin{lem}[\cite{BDE+_2018}, Lemma 2.6]
\label{lem:subgauss_def}
Let $X$ be a $\tau$-subgaussian random variable.
Then for all $t>0$,
\begin{equation}
    \label{eq:subgauss_def}
    \Pr[X>t] \leq \exp\left(-\frac{t^2}{2\tau^2}\right).
\end{equation}
\end{lem}

\begin{lem}[\cite{BDE+_2018}, Lemma 2.7]
\label{lem:subgauss_sum}
Let $X_1,\ldots,X_n$ be independent random variables such that $X_i$ is $\tau_i$-subgaussian.
For all $\mu_1,\ldots,\mu_n\in\R$, the random variable $X=\mu_1X_1+\cdots+\mu_nX_n$ is $\tau$-subgaussian, where
\[
\tau^2 = \mu_1^2\tau_1^2 + \cdots + \mu_n^2\tau_n^2.
\]
\end{lem}

A random vector $\mbf{x}\in\R^n$ is called $\tau$-subgaussian if for all unit vectors $\mbf{u}\in\R^n$, the inner product $\ang{\mbf{u},\mbf{x}}$ is a $\tau$-subgaussian random variable.
By this definition, a random vector $\mbf{x}$ that has components $x_i$ that are all independent $\tau$-subgaussian random variables is $\tau$-subgaussian.
Similarly to subgaussian random variables, subgaussian vectors also have strong tail bounds.

\begin{lem}
\label{lem:subgaussian_vector_bounds}
Let $\mbf{v}$ be a $\tau$-subgaussian random vector in $\R^n$.
Then
\[
\Pr[\norm{\mbf{v}} \geq t] \leq 2n\exp\left(-\frac{t^2}{2\tau^2 n}\right).
\]
\end{lem}
\begin{proof}
$\norm{\mbf{v}} \geq t$ only if at least one of its components $v_i$ satisfies $|v_i|\geq t/\sqrt{n}$.
However, $v_i$ can be written as the inner product $\ang{\mbf{v},e_i}$, where $e_i$ is the $i$-th standard basis vector.
Similarly, $-v_i = \ang{\mbf{v},-e_i}$.
Since the standard basis vectors are unit vectors in $\R^n$, and since $\mbf{v}$ is $\tau$-subgaussian, this means that each of the components $v_1,\ldots, v_n, -v_1,\ldots, -v_n$ is $\tau$-subgaussian.
Fixing $s=t/\sqrt{n}$, we can use Eq. \ref{eq:subgauss_def} to get
\begin{align*}
    \Pr[\norm{\mbf{v}} \geq t] &\leq \Pr\big[|v_1|\geq s\big] + \cdots + \Pr\big[|v_n| \geq s\big]\\
    & \leq 2n\exp\left(-\frac{t^2}{2\tau^2 n}\right).
\end{align*}
\end{proof}

\begin{lem}[\cite{BDE+_2018}, Lemma 2.9]
\label{lem:subgaussian_matrix}
Let $\mbf{x}$ be a $\tau$-subgaussian random vector in $\R^n$ and $\mbf{A}\in\R^{m\times n}$.
Then $\mbf{y}=\mbf{A}\mbf{x}$ is a $\tau'$-subgaussian random vector in $\R^m$, with $\tau'=\tau\cdot\opnorm{\mbf{A}^\msf{T}}$.
\end{lem}

\subsection{Physics of the PUF}
\label{subsec:physics}
In the absence of nonlinear optical effects, the behavior of the PUF is governed by the linear wave equation
\begin{equation}
    \label{eq:linear_wave}
    \left[\nabla^2 -\frac{1}{c^2}\frac{\partial^2}{\partial t^2}\ve(\mbf{r})\right]\Psi(\mbf{r},t)=J(\mbf{r},t),
\end{equation}
where $\ve(\mbf{r})$ is the dielectric of the scattering token at a position $\mbf{r}$, which encodes values of the dielectric of the glass beads used as scatterers, as well as the dielectric inside the optical resin \cite{SM_2000}.
The resin and the scatterers are both assumed to be locally isotropic, meaning that their dielectric coefficients are independent of the direction of polarization.
$\Psi(\mbf{r},t)$ is a complex scalar field which encodes the amplitude and phase of the electric field at a position $(\mbf{r})$ and a time $t$.
Finally, $J(\mbf{r},t)$ is a monochromatic source term such that $J(\mbf{r},t) = J_0(\mbf{r})\exp(-i\omega_0t)$ and $\Psi(\mbf{r},t) = \psi(\mbf{r})\exp(-i\omega_0t)$.
Eq. \ref{eq:linear_wave} can then be rewritten as
\begin{align*}
    \left[\nabla^2 + \frac{\omega_0^2}{c^2}\ve(\mbf{r})\right]\psi(\mbf{r}) &= J_0(\mbf{r}),
\end{align*}
where $J_0(\mbf{r})$ is the amplitude of the source term at a given location, $\psi(\mbf{r})$ is the amplitude of the electric field, and $\omega_0$ is the angular frequency of the source.
Given the linearity of Eq. \ref{eq:linear_wave}, if the PUF receives challenges $c_1$ and $c_2$ and gives responses $r_1$ and $r_2$, respectively, then if it receives the challenge $c_1+c_2$, the corresponding response will be $r_1+r_2$.

Nonlinear optical effects occur in all optical media, but they are usually insignificant if the magnitude of the electromagnetic field is much smaller than the fields within the molecules and atoms of the material.
When incident light is of a sufficient intensity in a nonlinear medium, the polarization of the medium begins to depend non-linearly on the electromagnetic fields.
For media that are locally isotropic, this nonlinearity means the index of refraction depends on the intensity of the transmitted electromagnetic fields \cite{new_2011}.
This gives the nonlinear wave equation
\[
\left[\nabla^2 + \frac{\omega_0^2}{c^2}\ve\left(\mbf{r},|\psi(\mbf{r})|^2\right)\right]\psi(\mbf{r}) = J_0(\mbf{r}).
\]
In general, $\ve$ can be written as a power series in the field intensity $|\psi(\mbf{r})|^2$.
The nonlinear wave equation can thus be rewritten according to \cite{SM_2000, Eaton_1991} as
\begin{equation}
    \label{eq:nonlinear_wave}
    \left[\nabla^2 + \frac{\omega_0^2}{c^2}\sum_{k=0}^\infty\ve_k(\mbf{r})|\psi(\mbf{r})|^{2k}\right]\psi(\mbf{r}) = J_0(\mbf{r}).
\end{equation}
In the limit as the nonlinear effects go to 0, such as if the medium has weak nonlinear properties or if the laser in the PUF is being run at lower intensities such that all the nonlinear effects are small, the nonlinear component can be truncated after the $\ve_0(\mbf{r})$ term, and Eq. \ref{eq:nonlinear_wave} is equivalent to Eq. \ref{eq:linear_wave}.
For stronger nonlinearity or very high laser intensities, more terms of the power series are necessary, though the nonlinear terms are small corrections except for in very extreme cases, as each successive $\ve_k$ term is typically much smaller than the one before it \cite{new_2011, Eaton_1991}.

\subsection{Learning With Errors}
\label{subsec:lwe}
Learning With Errors (LWE) is a computational problem that has been used as a basis for the security of various candidate post-quantum encryption schemes in lattice-based cryptography \cite{Regev_2005, LPV_2013, BNM+_2022}.
In LWE, one is tasked with learning a secret vector $\mbf{s}\in\Z_p^n$ given polynomially many pairs $(\mbf{a}_i, b_i)\in \Z_p^{n+1}$, where $b_i = \ang{\mbf{a}_i,\mbf{s}}+e_i\mod p$, the $\mbf{a}_i$ are uniformly distributed in $\Z_p^n$, and the $e_i$ are sampled from a discrete Gaussian distribution on $\Z_p$.
It was shown in \cite{Regev_2005} that properly parameterized LWE is at least as hard as several worst-case variants of lattice problems such as the Shortest Independent Vectors Problem (SIVP), and the Gap Shortest Vector Problem (GapSVP), which are conjectured to be hard for both classical and quantum computers.

Continuous Learning With Errors (CLWE) was introduced in \cite{BRST_2021} as a continuous variant of LWE, with quantum reductions from the same lattice problems (SIVP, GapSVP, etc.) that underlie the hardness of LWE.
Later, the authors \cite{GVV_2022} demonstrated polynomial-time reductions between LWE and CLWE, showing that the two problems are equivalently hard.
In CLWE$_{\beta,\gamma}$, for parameters appropriate $\beta,\gamma > 0$, one needs to find a secret unit vector $\mbf{s}\in \R^n$ given polynomially many pairs of the form $(\mbf{a}_i,\mbf{b}_i)\in\R^{n+1}$, where $b_i = \gamma\ang{\mbf{a}_i, \mbf{s}} + e_i\mod 1$, the $\mbf{a}_i$ are distributed according to a continuous Gaussian distribution in $\R^n$ with covariance matrix $I_n/(2\pi)$, and the error terms $e_i$ are sampled from a continuous Gaussian distribution on $\R$ with variance $\beta^2/(2\pi)$.

\subsection{PAC-Learning}
\label{subsec:PAC-learning}
The Probably Approximately Correct (PAC) framework is a general model for evaluating the learnability of classes of functions first described in \cite{Valiant_1984}.
The general idea behind PAC learning is that in order to successfully learn a target concept or function, one should, with high probability, produce a hypothesis that is a good approximation of the target concept.
PAC learning has previously been used to prove the theoretical learnability of various PUF designs \cite{GTSS_2015,GTS_2016,GTFS_2016,GTFS_2017,Ganji_2018,CMH_2020}.
In this work, we use the agnostic PAC framework described in \cite{MRT_2018} to define PAC-learnability as follows:

A class of functions $\mathcal{H}:X\to Y$, called the hypothesis class, is said to be \textit{PAC-learnable} if there exists an algorithm $\mathcal{A}$ such that, for all $\ve > 0$ and $\delta\in (0,1)$, and any target concept $h_0\in H$, then with a set $S$ of $m=O(\mathsf{poly}(1/\ve,1/\delta,n))$ samples drawn according to a distribution $\mathcal{D}$ on $X\times Y$, the algorithm $\mathcal{A}$ will output a hypothesis $h_S:X\to Y$ such that
\[
\Pr_{S\sim \mathcal{D}}\left[R(h_S)-\inf_{h\in \mathcal{H}}R(h)\leq\ve\right]\geq 1-\delta,
\]
according to some generalization error, or risk function $R$.
If the algorithm also terminates in $O(\mathsf{poly}(1/\ve,1/\delta,n))$ time, then it is called an efficient PAC learning algorithm.

In our case, since we want to learn PUFs that essentially encode linear systems, the functions in the hypothesis class are just linear functions in $n$ variables.
Since linear functions in can be encoded as inner products of coefficient vectors $\mbf{h}$ and variable vectors $\mbf{x}$, we will set the risk function $R(h)$ to be the maximum difference between the value $\ang{\mbf{h},\mbf{x}}$ of the hypothesis function and $\ang{\mbf{h}_0,\mbf{x}}$, the value of the target concept.
Thus, the PAC condition can be rewritten as
\begin{equation}
    \label{eq:pac_definition}
    \Pr_{S\sim \mathcal{D}}\left[\max_{\mbf{x}\in X}\left|\ang{\mbf{h}-\mbf{h}_0,\mbf{x}}\right|\leq\ve\right]\geq 1-\delta.
\end{equation}
As we will show in Section \ref{sec:learning}, a simple linear regression algorithm can provably efficiently PAC-learn the PUF, under the mild assumption that the error distribution is subgaussian or can be shifted by a constant offset to produce a subgaussian distribution.

In order for a PUF design to be secure against polynomially bounded adversaries, it cannot be efficiently PAC-learned.
In other words, any algorithm that satisfies the PAC condition should either require exponentially many (in $1/\ve$, $1/\delta$, or $n$) samples or terminate after an exponentially long time.
As mentioned in Section \ref{subsec:lwe}, appropriately parameterized LWE and CLWE are conjectured to be hard to solve under hardness assumptions for worst-case lattice problems \cite{Regev_2005, BRST_2021}.
Thus, under those hardness assumptions, they cannot be efficiently PAC-learned since any algorithm that could efficiently PAC-learn LWE or CLWE would be able to solve those worst-case lattice problems in polynomial time.

\section{Learning Optical PUF behavior}
\label{sec:learning}
Throughout this section, we will assume that the distribution of measurement noise in the PUF responses is subgaussian.
Any nonzero mean in the noise terms will appear as a constant term that can be discarded at the end of the learning algorithm.
If the noise is sampled from a distribution with unbounded support, we can choose to reject samples with too large of noise.
By forcing all the responses to have bounded noise, Lemma \ref{lem:hoeffding} ensures that the noise distribution either is subgaussian or can be shifted by a constant offset to give a subgaussian distribution.

In Section \ref{subsec:nonlinear_learning}, we perform a perturbative analysis for the PUF responses within a weakly nonlinear regime, where terms of quadratic and higher order in the nonlinear correction are considered negligible.
This type of analysis implicitly assumes that the PUF responses are dominated by linear effects, with only a few low-degree nonlinear terms that make up a small correction.
This is true for optical PUFs containing lasers of low power or using materials that have weak nonlinear optical properties, such that the magnitude of the optical electromagnetic field from the laser is much smaller than the fields within the molecules and atoms of the material, and thus can be treated as a small perturbation to the linear behavior \cite{new_2011}.

\subsection{Learning Non-Integrated Optical PUFs}
\label{subsec:non-integrated_learning}

\begin{figure}[!t]
    \centering
    \includegraphics[width=0.5\linewidth]{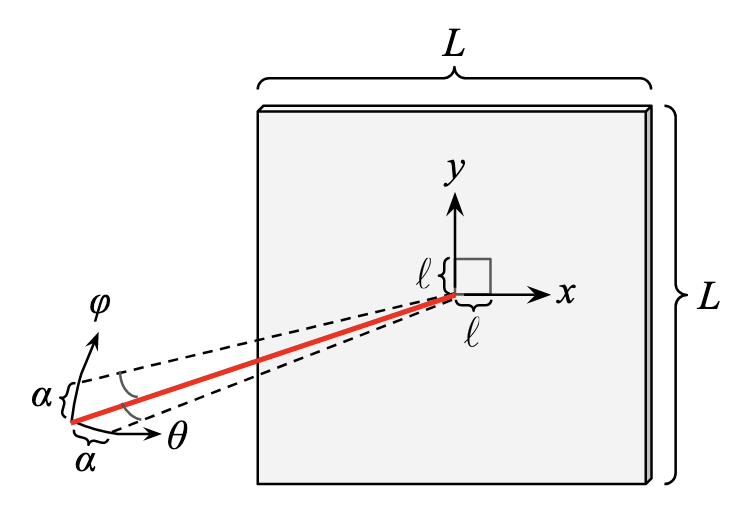}
    \vspace{-0.25cm}
    \caption{In a non-integrated optical PUF, the laser's position $(x,y)$ and direction of incidence $(\theta,\phi)$ can be varied as aprt of the challenge.
    Positional ($\pm \ell$) and angular ($\pm\alpha$) uncertainty in the alignment system means that the number of distinct challenge orientations scales polynomially with the physical size of the scattering token and the precision of the alignment system.}
    \label{fig:non-integrated_alignment}
\end{figure}

A non-integrated optical PUF, such as the original optical PUF in \cite{Pappu_2001}, allows for the $(x,y)$ position and $(\theta,\phi)$ angular orientation of the laser to be changed relative to the scattering token as part of the challenge (Figure \ref{fig:non-integrated_alignment}).
Given a challenge position and angle $(x,y,\theta,\phi)$, assume that uncertainty in the alignment system causes the actual position and angle of the laser to vary by up to $\pm\ell$ and $\pm\alpha$, respectively.
Thus, in order for a particular laser orientation to correspond to a unique challenge, the positions of each challenge need to be separated by a spatial distance of at least $2\ell$ in the $x$ and $y$ directions, and by an angular distance of at least $2\alpha$ in the $\theta$ and $\phi$ directions.
Thus, if the scattering token can be illuminated over a surface area $A = L^2$, with an angle of incidence anywhere on a hemisphere ($\theta, \phi \in [0,\pi]$), the number of distinct orientations of the laser relative to the scattering token is bounded above by
\[
\# \textrm{ of distinct orientations} \leq \frac{\pi^2L^2}{16\alpha^2\ell^2},
\]
which scales polynomially in the physical size $L$ of the token, as well as in the spatial and angular precisions $1/\ell$ and $1/\alpha$ of the alignment system.

Because the position of the light source is fixed relative to the scattering token in an integrated PUF, learning the behavior of an integrated PUF is equivalent to learning the behavior of a particular challenge position and orientation of a non-integrated PUF that uses the same scattering token.
In particular, this implies that any algorithm that learns an integrated optical PUF in polynomial time can be extended to learn a non-integrated optical PUF in polynomial time simply by applying that algorithm for each of the polynomially many orientations of the non-integrated PUF.

\subsection{Linear Scattering Media}
\label{subsec:linear_learning}
A challenge to the PUF consists of a specific pattern on the LCD mask, which determines what parts of the PUF medium are illuminated by the laser (Fig. \ref{fig:PUF_schematic}).
We can describe the $j$-th pixel in a particular challenge image on the mask by a real number $b_j$ between 0 and 1 that describes what proportion of the incident radiation gets transmitted through that pixel.
A challenge $\mbf{b}$ to the PUF can then be written as a vector $\mbf{b} = (b_1,\ldots, b_N) \in [0,1]^N$, where $N$ is the number of pixels in the LCD screen.

At a given pixel in the detector, the complex amplitude $a$ of the electric field can be written as a function $a(\mbf{b})$ If the PUF medium is linear, $a(\mbf{b})$ can be written as a linear function
\[
a(\mbf{b}) \,=\, a(b_1,\ldots,b_N) \,=\, \sum_{j=1}^Nb_jt_j,
\]
where the $t_j$ are complex transmission coefficients that encode how the amplitude and phase of the light passing through pixels $b_j$ is transmitted to that part of the detector.
If the speckle pattern is picked up with a charge-coupled device (CCD) or a similar camera chip, then the response $f_{\msf{PUF}}(\mbf{b})$ measures the intensity $|a|^2$ of the laser light at that location, so it is quadratic in the $b_j$:
\[
f_{\msf{PUF}}(\mbf{b})\,=\, |a(\mbf{b})|^2 \,=\, \sum_{j=1}^N\sum_{k=1}^N b_jb_kt_jt_k^*,
\]
where $t_k^*$ denotes the complex conjugate of $t_k$.
We can define the new vectors $\mbf{c}=(1,c_1,\ldots,c_n)\in[0,1]^{n+1}$ and $\mbf{s}=(s_0,s_1,\ldots,s_n)\in[0,1]^{n+1}$ such that $f_{\msf{PUF}}(\mbf{c}) = \ang{\mbf{c},\mbf{s}}$, where each component $c_i$ is a monomial of total degree at most $2$ in the $b_j$, and where the first component in $\mbf{c}$ and $\mbf{s}$ representing a constant offset.
For an adversary to successfully learn the PUF, they will need to determine an approximate candidate vector $\hat{\mbf{s}}$ such that $|\ang{\mbf{c},\mbf{s}} - \ang{\mbf{c}, \hat{\mbf{s}}}|<\ve$.
In other words, they want to be able to approximate the PUF's behavior to within $\ve$ for any possible challenge $\mbf{c}$.

The problem of learning the PUF can thus be written as a problem of determining $\hat{\mbf{s}}$ from noisy CRPs.
For any given challenge $\mbf{c}_i$, the adversary will have access to the pair $(\mbf{c}_i, \ang{\mbf{c}_i,\mbf{s}} + e_i)$, where without loss of generality, $e_i$ is a $\tau_e$-subgaussian random noise term, which could, for example, arise from random measurement error or random fluctuations in the transparency of the pixels in the LCD.
If the measurement noise $e_i$ has nonzero mean, then that will show up in the $s_0$ constant term, which we can throw out at the end.
If the noise is sampled from a distribution with unbounded support, we can choose to reject samples with too large of noise.
In particular, given $\alpha>0$ such that $\Pr[|e_i|<\alpha]>1/2$, we can reject samples that we know have $|e_i|>\alpha$ and then use the same analysis as for distributions with bounded support.
In this case we will need (with overwhelming probability) around twice as many CRPs as we would otherwise, and Eqs. \ref{eq:puf_learnability}--\ref{eq:nl_time_complexity} will all pick up an extra factor of $M$ since for a given challenge $\mbf{c}_i$, the error $|e_i|$ may not be simultaneously less than $\alpha$ across all $M$ pixels in the CCD.

We can express a PUF response $r_i$ as
\[
r_i = \ang{\mbf{c}_i,\mbf{s}} + e_i,
\]
and we can combine the expressions for a set of $m$ CRPs to get
\[
\mbf{Cs}+\mbf{e}=\mbf{r},
\]
where $\mbf{c}_i$ is the $i$-th row of the $m\times n$ matrix $\mbf{C}$, and likewise for the error and response vectors $\mbf{e}$ and $\mbf{r}$.
While the pairs $(\mbf{c}_i, r_i)$ appear to be similar to samples generated for LWE or CLWE, they are are not subject to modular reduction, which removes key information about the $r_i$ samples that can otherwise be leveraged to learn $\mbf{s}$, as described in \cite{BDE+_2018}.

In order to learn $\mbf{s}$, we produce an estimate $\hat{\mbf{s}}$ ignores the error vector $\mbf{e}$ such that $\mbf{C}\hat{\mbf{s}}\approx\mbf{r}$.
Assuming that $\mbf{C}^{\msf{T}}\mbf{C}$ is invertible (and we will provide a condition for this to be true), this is done by solving for $\hat{\mbf{s}}$, giving the least-squares estimate
\[
\hat{\mbf{s}} = \left(\mbf{C}^{\msf{T}}\mbf{C}\right)^{-1}\mbf{C}^{\msf{T}}\mbf{r}.
\]
Once we have our estimate, we can now bound the estimation error $\ve$ between a legitimate PUF response $\ang{\mbf{c},\mbf{s}}$ and the approximate PUF response $\ang{\mbf{c},\hat{\mbf{s}}}$.
Since $\mbf{Cs}+\mbf{e}=\mbf{r}$, we get the relation
\begin{equation}
\label{eq:error_from_secret_diff}
    \hat{\mbf{s}}-\mbf{s} = \left(\mbf{C}^{\msf{T}}\mbf{C}\right)^{-1}\mbf{C}^{\msf{T}}\mbf{e},
\end{equation}
which by Lemma \ref{lem:subgaussian_matrix} is a $\tau'$-subgaussian random vector, where
\[
\tau'=\tau_e\cdot\opnorm{\left(\mbf{C}^{\msf{T}}\mbf{C}\right)^{-1}\mbf{C}^{\msf{T}}} = \tau_e\cdot\opnorm{\mbf{M}},
\]
where $\mbf{M} = \left(\mbf{C}^{\msf{T}}\mbf{C}\right)^{-1}\mbf{C}^{\msf{T}}$.
By Eq. \ref{eq:opnorm}, this is equal to
\begin{equation}
    \label{eq:tau_prime_def}
    \tau' = \tau_e\sqrt{\lambda_{\max}\left(\mbf{MM}^{\msf{T}}\right)} = \frac{\tau_e}{\sqrt{\lambda_{\min}\left(\mbf{C}^{\msf{T}}\mbf{C}\right)}}.
\end{equation}
The matrix $\mbf{C}^{\msf{T}}\mbf{C}$ can be written as
\[
\sum_{i=1}^m\mbf{c}_i^\msf{T}\mbf{c}_i,
\]
a sum of $m$ outer product matrices, one for each challenge.
By Lemma \ref{lem:outer_product_eigenvalues}, we can see that each of these matrices has exactly one nonzero eigenvalue equal to $\norm{\mbf{c}_i}^2$.

\begin{lem}
\label{lem:outer_product_eigenvalues}
For any row vector $\mbf{x}\in\R^n$, the eigenvalues of the outer product matrix $\mbf{x}^\msf{T}\mbf{x}$ are $\norm{\mbf{x}}^2$ and $0$.
\end{lem}
\begin{proof}
First note that if $\norm{\mbf{x}} = 0$, then $\mbf{x}^\msf{T}\mbf{x}$ is just the zero matrix, which only has eigenvalue $0$. Assume that $\norm{\mbf{x}} > 0$, and let $\mbf{u}\in \R^n$ be a nonzero eigenvector of $\mbf{x}^\msf{T}\mbf{x}$.
Then $\mbf{x}^\msf{T}\mbf{x}\mbf{u} = \lambda\mbf{u}$ for some $\lambda\in\C$.
If $\mbf{x}\mbf{u} = 0$, then we have that $\mbf{x}^\msf{T}\mbf{x}\mbf{u} = \mbf{x}^\msf{T}\cdot 0 = \mbf{0} = \lambda\mbf{u}$.
Since $\norm{\mbf{u}} > 0$, we know that $\lambda = 0$.
If $\mbf{x}\mbf{u}\neq 0$, multiplying on both sides by $\mbf{x}$ gives $\mbf{x}\mbf{x}^\msf{T}\mbf{x}\mbf{u} = \mbf{x}\lambda\mbf{u}$.
However, $\mbf{x}\mbf{x}^\msf{T} = \norm{\mbf{x}}^2$, and $\lambda$ commutes with $\mbf{x}$ on the right side giving $\norm{\mbf{x}}^2\mbf{x}\mbf{u} = \lambda\mbf{x}\mbf{u}$, from which it follows that $\lambda = \norm{\mbf{x}}^2$.
\end{proof}

Outer products of real vectors are always real and symmetric.
In addition, since none of their eigenvalues are negative by Lemma \ref{lem:outer_product_eigenvalues}, the $\mbf{c}_i^\msf{T}\mbf{c}_i$ are positive semidefinite.
The maximum eigenvalue of these matrices is $\lambda_{\max} = \norm{\mbf{c}_i}^2$.
Since $\mbf{c}$ has $n$ components, each within the interval $[0,1]$, we know that $\norm{\mbf{c}}^2\leq n$.
This combination of properties (real symmetric, positive semidefinite, and bounded maximum eigenvalue) allows us to use a matrix Chernoff bound to find a bound on the minimum eigenvalue of their sum.

\begin{prop}[Matrix Chernoff II \cite{Tropp_2011}]
\label{prop:matrix_chernoff_2}
Consider a finite sequence $\{\mbf{A}_i\}_{i=1}^m$ of independent, random, symmetric, and positive semi-definite matrices of dimension $d$ that satisfy $\lambda_{\max}(\mbf{A}_i)\leq R,$ for some $R\geq 0$.
Compute the minimum eigenvalue of the sum of expectations:
\[
\mu_{\min} := \lambda_{\min}\left(\sum_{i=1}^m \E[\mbf{A}_i]\right).
\]
Then
\[
\Pr\left[\lambda_{\min}\left(\sum_{i=1}^m \mbf{A}_i\right)\leq (1-\alpha)\mu_{\min}\right] \leq d\exp\left(-\frac{\alpha^2\mu_{\min}}{2R}\right)
\]
for all $\alpha\in [0,1]$.
\end{prop}

To determine $\mu_{\min}$, first note that since all the $\mbf{c}$ are identically and independently distributed, their expectation is the same.
Thus, we have that
\begin{equation}
    \label{eq:mu_min_def}
    \mu_{\min} = \lambda_{\min}\left(\sum_{i=1}^m\E\left[\mbf{c}^{\msf{T}}\mbf{c}\right]\right) = m\cdot \lambda_{\min}\left(\E[\mbf{c}^{\msf{T}}\mbf{c}]\right).
\end{equation}
Since $\E[\mbf{c}^{\msf{T}}\mbf{c}]$ is a real symmetric matrix, by the spectral theorem there exists an orthogonal matrix $\mbf{P}$ such that $\mbf{P}^{\msf{T}}\E[\mbf{c}^{\msf{T}}\mbf{c}]\mbf{P}$ is diagonal.
Since the expectation operator is linear, this means that $\E\left[(\mbf{cP})^\msf{T}\mbf{cP}\right]$ is diagonal, and that the eigenvalues of $\E[\mbf{c}^{\msf{T}}\mbf{c}]$ are $\lambda_j = \E[(\mbf{cP})_j^2]$.
Using $\mbf{P}$, we can rewrite an individual response $r_i$ as
\[
r_i = \ang{\mbf{c}_i\mbf{P},\mbf{P}^{\msf{T}}\mbf{s}} + e_i,
\]
with the matrix expression for $m$ responses
\[
\mbf{r} = \mbf{CPP}^{\msf{T}}\mbf{s} + \mbf{e}.
\]
If there exists some $j$ such that $\lambda_j=0$, then for any challenge $\mbf{c}_i$, the component $(\mbf{c}_i\mbf{P})_j = 0$, meaning that $f_{\msf{PUF}}$ is independent of the specific value of the $j$-th component of $\mbf{P}^{\msf{T}}\mbf{s}$.
Thus, we can instead work with the challenges $\tilde{\mbf{c}}_i = \mbf{c}_i\mbf{P}$ and $\tilde{\mbf{s}} = \mbf{P}^{\msf{T}}\mbf{s}$, where the $j$-th components corresponding to eigenvalues $\lambda_j=0$ are removed.
Let $\tilde{\mbf{C}}$ be the matrix with $j$-th row $\tilde{\mbf{c}}_j$, and compute the estimate $\hat{\tilde{\mbf{s}}}$ by taking
\[
\hat{\tilde{\mbf{s}}} = \left(\tilde{\mbf{C}}^{\msf{T}}\tilde{\mbf{C}}\right)^{-1}\tilde{\mbf{C}}^{\msf{T}}\mbf{r}.
\]
After obtaining $\hat{\tilde{\mbf{s}}}$, we can replace the removed indices $\hat{\tilde{s}}_j$ with any number and left multiply by $\mbf{P}$ to obtain $\hat{\mbf{s}}$ as before, where for any challenge $\mbf{c}$, we have $\ang{\mbf{c},\mbf{s}}=\ang{\tilde{\mbf{c}}, \tilde{\mbf{s}}}$, and likewise for the estimate.
By switching to using $\tilde{\mbf{c}}_i$, we can ensure that the expected outer product is diagonal and has a nonzero minimum eigenvalue.
Since the eigenvalues of orthogonal matrices all have modulus 1, and since $\tilde{\mbf{c}}$ has at most as many components as $\mbf{c}$, we can still fix $R=n$ since $\norm{\tilde{\mbf{c}}}\leq \norm{\mbf{c}}$.
Let $\xi = \lambda_{\min}\left(\E[\tilde{\mbf{c}}^{\msf{T}}\tilde{\mbf{c}}]\right)$ such that $\mu_{\min} = m\xi$ in Eq. \ref{eq:mu_min_def}.

Setting $\alpha = 1/2$ in Proposition \ref{prop:matrix_chernoff_2}, we can bound the minimum eigenvalue of $\tilde{\mbf{C}}^{\msf{T}}\tilde{\mbf{C}}$ by
\begin{equation}
    \label{eq:prob_bound}
    \Pr\left[\lambda_{\min}\left(\tilde{\mbf{C}}^{\msf{T}}\tilde{\mbf{C}}\right)\leq \frac{m\xi}{2}\right]\leq n\exp\left(-\frac{m\xi}{8n}\right).
\end{equation}
If we want to pick $m$ such that the probability in Eq. \ref{eq:prob_bound} is less than or equal to $\exp(-\eta)$, for $\eta>0$, then it suffices to pick $m$ such that
\begin{equation}
    \label{eq:m_bound_1-1}
    m\geq \frac{8n}{\xi}(\eta + \ln n).
\end{equation}
So, if Eq. \ref{eq:m_bound_1-1} is satisfied, we know that $\tilde{\mbf{C}}^{\msf{T}}\tilde{\mbf{C}}$ is invertible, and we have from Eq. \ref{eq:tau_prime_def} that, with probability at least $1-\exp(-\eta)$,
\[
\tau' = \tau_e\sqrt{\frac{2}{m\xi}}.
\]
In this case, by Lemma \ref{lem:subgaussian_vector_bounds}, we have that
\begin{equation}
    \label{eq:prob_bound_2}
    \Pr\left[\norm{\tilde{\mbf{s}}-\hat{\tilde{\mbf{s}}}}\geq \frac{\ve}{\sqrt{n}} \right]\leq 2n\exp\left(-\frac{\ve^2m\xi}{4n^2\tau_e^2}\right).
\end{equation}
If we pick $m$ such that the probability in Eq. \ref{eq:prob_bound_2} is less than or equal to $\exp(-\eta)$, then it suffices to pick $m$ such that
\begin{equation}
    \label{eq:m_bound_1-2}
    m\geq \frac{4n^2\tau_e^2}{\ve^2\xi}(\eta + \ln(2n)).
\end{equation}
Taking Eqs. \ref{eq:m_bound_1-1} and \ref{eq:m_bound_1-2} into account, we can see that if we set
\begin{equation}
    \label{eq:m_bound_1}
    m \geq \max\left\{\frac{8n}{\xi}(\eta + \ln n),\, \frac{4n^2\tau_e^2}{\ve^2\xi}(\eta + \ln(2n))\right\},
\end{equation}
then we know that, for any challenge $\mbf{c}\in[0,1]^n$,
\[
\left|\ang{\mbf{c},\mbf{s}} - \ang{\mbf{c}, \hat{\mbf{s}}}\right| = \left|\ang{\tilde{\mbf{c}}, \tilde{\mbf{s}}-\hat{\tilde{\mbf{s}}}}\right|\leq\norm{\tilde{\mbf{c}}}\norm{\tilde{\mbf{s}}-\hat{\tilde{\mbf{s}}}} \leq \ve.
\]
Thus, $\left|\ang{\mbf{c},\hat{\mbf{s}}}-f_{\msf{PUF}}\right|\leq\ve$, with probability at least $(1-\exp(-\eta))^2$.
Thus, the probability of simultaneously predict $f_{\msf{PUF}}$ to within $\ve$ for all $M$ pixels in the CCD is at least $(1-\exp(-\eta))^{2M}$.
If we want to achieve a good estimate with probability at least $1-\delta$, for $\delta\in (0,1)$, then since 
\[
(1-\exp(-\eta))^{2M}\geq 1-2M\exp(-\eta)
\]
for all $\eta > 0$, then to have $(1-\exp(-\eta))^{2M}\geq 1-\delta$, it suffices to fix
\[
\eta \geq \ln\left(\frac{2M}{\delta}\right).
\]
Substituting this value of $\eta$ into Eq. \ref{eq:m_bound_1} implies that it suffices to fix
\begin{equation}
    \label{eq:m_bound_2}
    m \geq \max\left\{\frac{8n}{\xi}\ln\left(\frac{2Mn}{\delta}\right),\, \frac{4n^2\tau_e^2}{\ve^2\xi}\ln\left(\frac{4Mn}{\delta}\right)\right\}.
\end{equation}
Since $n = O(N^2)$, Eq. \ref{eq:m_bound_2} gives an asymptotic bound on the required number of CRPs of
\begin{equation}
    \label{eq:puf_learnability}
    m = O\left(\frac{N^4\tau_e^2}{\ve^2\xi}\ln\left(\frac{MN^2}{\delta}\right)\right).
\end{equation}
In order to obtain $\hat{\mbf{s}}$, we need to compute the product
\[
\hat{\tilde{\mbf{s}}} = \left(\tilde{\mbf{C}}^{\msf{T}}\tilde{\mbf{C}}\right)^{-1}\tilde{\mbf{C}}^{\msf{T}}\mbf{r},
\]
which has time complexity $O(n^2m)$ with basic matrix multiplication.
Computation of the inverse $\left(\tilde{\mbf{C}}^{\msf{T}}\tilde{\mbf{C}}\right)^{-1}$ requires $O(n^3)$ time using Gaussian elimination, as does diagonalization of $\E[\mbf{c}^\msf{T}\mbf{c}]$ using a singular value decomposition \cite{HGI_2007}.
Thus, the overall time complexity for learning the PUF for all $M$ pixels in the speckle pattern is asymptotically given by
\begin{equation}
    \label{eq:puf_time_complexity}
    O\left(\frac{N^8\tau_e^2}{\ve^2\xi}\ln\left(\frac{MN^2}{\delta}\right)\right),
\end{equation}
which is polynomially bounded in $N$, $M$, $\ve$, and $\delta$.
In particular, this means that the PUF is efficiently PAC-learnable if it uses linear scattering media.

It should be noted that the approach here cannot be used to solve appropriately implemented instances of LWE or CLWE.
In particular, from Eq. \ref{eq:error_from_secret_diff}, we can see that the difference between the actual value for the secret $\mbf{s}$ and the least-squares estimate $\hat{\mbf{s}}$ multiplies the error by $\left(\mbf{C}^{\msf{T}}\mbf{C}\right)^{-1}\mbf{C}^{\msf{T}}$.
Because in LWE $\mbf{C}$ is sampled uniformly from $\Z_p^{m\times n}$, and all operations in LWE take place in $\Z_p$, this acts to magnify the error vector $\mbf{e}$, which leads to $\hat{\mbf{s}}-\mbf{s}$ being distributed according to very wide Gaussian distribution.
When reduced $\mod p$, this distribution becomes computationallly indistinguishable from the uniform distribution on $\Z_p$ \cite{Regev_2005}.
It is also clear that this approach cannot be applied CLWE since multiplicative inverses in $\R/\Z$ are not well-defined, so $\left(\mbf{C}^{\msf{T}}\mbf{C}\right)^{-1}$ cannot even be computed in principle.

\subsection{Nonlinear Scattering Media}
\label{subsec:nonlinear_learning}
Because nonlinear optical effects are generally small, we will analyze the case where the PUF contains a weakly nonlinear dielectric using a perturbative approach, which assumes that the characteristic size of the nonlinear effects is much smaller than the characteristic size of the linear effects, and that terms of quadratic or higher order in the small parameters are of negligible size.
In Eq. \ref{eq:nonlinear_wave}, we will simplify by moving the factor of $\omega_0^2/c^2$ into the $\ve_k$ terms.
Suppose that $\psi=\psi_{\msf{L}}+\delta\psi_{\msf{NL}}$ can be written as a linear term $\psi_{\msf{L}}$ and a small nonlinear term $\delta\psi_{\msf{NL}}$, where $\delta\psi_{\msf{NL}}\ll \psi_{\msf{L}}$ such that $|\psi|^k\approx|\psi_\msf{L}|^k(1+k\delta\psi_{\msf{NL}}/\psi_{\msf{L}})$, and where $\psi_{\msf{L}}$ solves the linear wave equation
\[
\left[\nabla^2 + \ve_0(\mbf{r})\right]\psi_{\msf{L}}(\mbf{r}) = J_0(\mbf{r}).
\]
Further, assume that the dielectric behaves mostly linearly, with $\ve=\ve_0+\delta\ve_{\msf{NL}}$, where again $\delta\ve_{\msf{NL}}\ll\ve_0$ with small measurable nonlinear effects up to degree $d$.
Cancelling terms quadratic in the small parameters gives
\[
\ve_0+\delta\ve_{\msf{NL}}=\ve_0+\sum_{k=1}^d\delta\ve_k|\psi|^{2k}\approx\ve_0+\sum_{k=1}^d\delta\ve_k|\psi_\msf{L}|^{2k}.
\]
Substituting into Eq. \ref{eq:nonlinear_wave} and simplifying by keeping only terms at most linear in the small parameters gives an expression for $\psi_{\msf{NL}}$ in terms of powers of $\psi_{\msf{L}}$:
\[
\left[\nabla^2 + \ve_0(\mbf{r})\right]\delta\psi_{\msf{NL}}(\mbf{r}) = -\sum_{k=1}^d\delta \ve_k(\mbf{r})\psi_{\msf{L}}(\mbf{r})\left|\psi_{\msf{L}}(\mbf{r})\right|^{2k}
\]
As we saw in the linear case, $\psi_{\msf{L}}$ can be written as a complex linear combination of the coefficients $b_j$.
Because $\psi$ is linear in the $b_j$, $|\psi_\msf{L}|^{2k}$ is a polynomial of degree $2k$ in the $b_j$, meaning that $\psi$ is a polynomial of degree $2d+1$ in the $b_j$.
Thus, $f_{\msf{PUF}}\approx |\psi|^2$ is a polynomial of degree $4d+2$ in the $b_j$.
From here, we can follow the same procedure as in the linear case by encoding the challenge vector $\mbf{c}$ which has $n=O(N^{4d+2})$ components, each of which is a monomial of total degree at most $4d+2$ in the $b_j$.
We can use the same bounds as before to get an asymptotic bound on the required number of CRPs of
\begin{equation}
    \label{eq:nl_sample_complexity}
    m = O\left(\frac{N^{8d+4}\tau_e^2}{\ve^2\xi}\ln\left(\frac{MN^{4d+2}}{\delta}\right)\right),
\end{equation}
as well as a time complexity bound of
\begin{equation}
    \label{eq:nl_time_complexity}
    O\left(\frac{N^{16d+8}\tau_e^2}{\ve^2\xi}\ln\left(\frac{MN^{4d+2}}{\delta}\right)\right).
\end{equation}
While these bounds grow much more quickly than for the linear case, they are still polynomial for a fixed value of $d$ (generally $d=1$ or 2 \cite{Eaton_1991}), so the PUF is still efficiently PAC-learnable.

\section{Conclusion}

\subsection{Results}

In Section \ref{sec:prelim}, we examined the underlying physics of integrated optical PUFs with masks and demonstrated that, with linear optics, the PUF acts as a quadratic polynomial of the challenge components $b_i$.
We introduced the PAC-learning framework, under which the task of learning the behavior of PUF in the presence of random noise, is equivalent to the problem of learning a noisy linear system in $O(N^2)$ dimensions.
By making this reduction, we were able to show in Section \ref{subsec:linear_learning} the convergence of a linear regression algorithm, based on mild assumptions about the noise distribution.
We found an asymptotic bound in Eq. \ref{eq:puf_learnability} for the number of CRPs required to learn the PUF behavior, based on the size $N$ of the LCD mask, the number of pixels $M$ in the speckle pattern detector, the accepted error $\ve$ in learning the PUF behavior, and the probability $1-\delta$ of learning the PUF, as well as the distributions of the challenge vectors and random sample noise.
The time complexity for a naive implementation of this algorithm was computed in Eq. \ref{eq:puf_time_complexity} to be
\[
O\left(\frac{N^8\tau_e^2}{\ve^2\xi}\ln\left(\frac{MN^2}{\delta}\right)\right).
\]
In particular, this means that optical PUFs with linear optics are efficiently PAC-learnable since they can be represented exactly by a polynomial.
Finally, in Section \ref{subsec:nonlinear_learning} we did a perturbative analysis of PUF designs containing dielectrics with nonlinear optical properties.
We showed that, under the assumption that the nonlinear effects were relatively small, the PUF still acts as a polynomial in the challenge components $b_i$, with the degree of the polynomial determined by the highest order of polarization susceptibility, and thus can be learned with access to polynomially many CRPs in polynomial time (Eqs. \ref{eq:nl_sample_complexity}, \ref{eq:nl_time_complexity}).

Since the computational complexity of the regression algorithm is polynomial, learning the PUF is not hard for an adversary with polynomially-bounded computational resources who has access to the challenges and noisy speckle data.
While the bounds given in Eqs. \ref{eq:puf_learnability}--\ref{eq:nl_time_complexity} grow very quickly with $N$, it should be noted they are generic polynomial bounds for a particular type of learning algorithm and are just intended to show that the optical PUFs considered are PAC-learnable with a polynomial sample and time complexity.
A more sophisticated analysis of the linear regression algorithm may provide tighter bounds, and more sophisticated learning approaches would likely require a much smaller sample set to learn the PUF in less time.

\subsection{Future Work}

In order for an integrated or non-integrated optical PUF to be plausibly secure against these types of adversaries, it cannot just use linear or weakly nonlinear scattering media.
To increase security, the raw speckle patterns could be cryptographically hashed, although this approach is susceptible to side-channel attacks if an adversary can avoid the hashing operation to access the raw speckle patterns.
In order to maintain security while avoiding a post-processing step, different PUF architectures or materials need to be used.
If alignment of the optical tokens is not an issue, the non-integrated optical PUFs described in \cite{Pappu_2001, PRTG_2002} were shown to be resilient to machine learning attacks by Support Vector Machines with linear kernels in \cite{RHU+_2013}.
However, the total number of CRPs in non-integrated optical PUFs only scales polynomially with the PUF size and alignment precision, which permits polynomial time read-out attacks, though such attacks may not be practically feasible due to limited read-out speed when aligning the PUF scattering tokens \cite{Pappu_2001}.

One possible approach that retains the integrated design is to dope the scatterers in linear optical systems with ``quantum dot'' materials such as those described in \cite{Eaton_1991}.
These are nanoparticles of semiconductor material that exhibit strong nonlinear properties at low light intensities.
Nonlinear optical systems are harder to model than linear systems since Eq. \ref{eq:nonlinear_wave}, the nonlinear wave equation governing the behavior of these systems, requires higher degree polynomials to approximate, making the task of learning the system much more difficult.
In addition, increasing the power of the laser will also increase the strength of the nonlinear effects and make the higher-order nonlinear terms more relevant, again increasing the required degree of a polynomial approximation.
Furthermore, if the nonlinear optical effects are comparable in size to the linear ones, the perturbative technique used in Section \ref{subsec:nonlinear_learning} is no longer applicable, meaning the PUF may be much harder to learn.

Another option is to use nonlinear materials that are not centrosymmetric such that their scattering properties are dependent on the polarization of the light passing through them \cite{Eaton_1991, new_2011}.
Because the dielectric constants of such materials are dependent on orientation, one must treat the electric field within the material as the laser propagates as a full vector field instead of a scalar field.
Furthermore, when using nonlinear non-centrosymmetric media, the perturbative technique in Section \ref{subsec:nonlinear_learning} gives an expression for the nonlinear term which contains a square root of a polynomial, meaning it cannot be reduced to a high degree linear system in the monomial terms like it could with isotropic materials.

In an ideal PUF design, one would embed a general case of an appropriately parameterized cryptographically hard problem within the PUF's behavior.
This approach is partially used in the Lattice PUF \cite{WXO_2019}; however all of the arithmetic required to implement such a cryptographic protocol should ideally be performed physically within the PUF structure itself, rather than just using the PUF to store a secret key.
If a PUF framework is designed with this methodology, in order for an adversary to learn an instance of the PUF, they need to solve a general case of the cryptographic hard problem.
Thus, either the adversary’s learning attack cannot run in polynomial time (as that would provide a general polynomial time solution to the cryptographic problem) or the hardness assumptions for that problem cannot hold.
In order to embed LWE or CLWE in an optical PUF, one would need to perform modular arithmetic operations directly within the optical system, which requires further research.
Modular reduction could also be achieved in a post-processing step; however any post-processing step opens up opportunities for side-channel attacks if an adversary can avoid it.

\section{Acknowledgements}
This research was supported by the Information Science and Technology Institute, the Nuclear Weapons Cyber Assurance Laboratory (NWCAL), and the Laboratory Directed Research and Development program of Los Alamos National Laboratory (LANL) under project numbers 20210529CR-IST and 20220800DI.
LANL is operated by Triad National Security, LLC, for the National Nuclear Security Administration of the U.S. Department of Energy (Contract No. 89233218CNA000001).
Approved for unlimited public release: LA-UR-23-29622.

\bibliographystyle{ieeetr}
\bibliography{refs.bib}

\begin{IEEEbiographynophoto}{Apollo Albright}
Apollo Albright is completing his undergraduate studies at Reed College in Portland, Oregon, USA, where he is majoring in mathematics and physics.
He is also an undergraduate research associate with the Analytics, Intelligence, and Technology Division of Los Alamos National Laboratory.
His research interests include classical and post-quantum cryptography, combinatorics, graph theory, and quantum and many-body physics. 
\end{IEEEbiographynophoto}

\begin{IEEEbiographynophoto}{Boris Gelfand}
Dr. Gelfand is a security researcher and systems engineer at Los Alamos National Labs and has many years' experience working as a contractor with DoD, DOE, and the IC.
Notably he was the chief designer and architect of the National Cyber Range and has been the PI of advanced research programs including many from DARPA.
He holds a PhD in computer science, as well as degrees in mathematics and physics.
Prior to coming to Los Alamos, he worked for Lockheed Martin in the Advance Technologies Laboratory.
\end{IEEEbiographynophoto}

\begin{IEEEbiographynophoto}{Michael Dixon}
Michael J. Dixon is a senior cyber security research scientist and principal investigator in LANL’s Advanced Research in Cyber Systems group and Nuclear Weapons Cyber Assurance Laboratory specializing in applied cryptography, secure machine learning and artificial intelligence, anti-tamper technologies, and provable security using formal methods.
Michael holds a Bachelor of Science and Engineering in Computer Science from the University of Michigan, College of Engineering, and attended MIT for graduate studies as an Advanced Study Program Fellow researching post-quantum and lattice-based cryptography.
\end{IEEEbiographynophoto}

\end{document}